\documentclass{article}

\usepackage{microtype}
\usepackage{graphicx}
\usepackage{subcaption}
\usepackage{booktabs} 
\usepackage{placeins}
\usepackage{stfloats}

\usepackage{hyperref}

\usepackage[accepted]{icml2026}
\usepackage[utf8]{inputenc}    
\usepackage[T1]{fontenc}       
\usepackage{lmodern}           

\usepackage{amsmath}
\usepackage{amssymb}
\usepackage{mathtools}
\usepackage{amsthm}
\usepackage{tcolorbox}

\usepackage[capitalize,noabbrev]{cleveref}

\theoremstyle{plain}
\newtheorem{theorem}{Theorem}[section]

\theoremstyle{definition}

\theoremstyle{remark}
\newtheorem{remark}[theorem]{Remark}

\usepackage[textsize=tiny]{todonotes}

\icmltitlerunning{Hessian Spectral Analysis at Foundation Model Scale}

\begin{document}
	

	\twocolumn[\icmltitle{Hessian Spectral Analysis at Foundation Model Scale}
	\begin{icmlauthorlist}
		\icmlauthor{Diego Granziol}{ox}
		\icmlauthor{Khurshid Juarev}{puredog}
	\end{icmlauthorlist}
	
	\icmlaffiliation{ox}{Mathematical Institute, University of Oxford, UK}
	\icmlaffiliation{puredog}{Purestrength AI, London, UK}
	\icmlcorrespondingauthor{Diego Granziol}{granziol@maths.ox.ac.uk}
	
	\printAffiliationsAndNotice{}]
	
	\begin{abstract}
Accurate Hessian spectra of foundation models have remained out of reach, leading most prior work to rely on small models or strong structural approximations. We show that faithful spectral analysis of the true Hessian is tractable at frontier scale. Using shard-local finite-difference Hessian vector products compatible with Fully Sharded Data Parallelism, we perform stochastic Lanczos quadrature on open-source language models with up to 100B parameters, producing the first large-scale spectral density estimates beyond the sub-10B regime. We characterize the numerical behavior of this pipeline, including finite-difference bias, floating-point noise amplification, and their effect on Krylov stability in fp32 and bf16, and derive practical operating regimes that are validated empirically. We further provide end-to-end runtime and memory scaling laws, showing that full-operator spectral probing incurs only a modest constant-factor overhead over first-order training. Crucially, direct access to the Hessian reveals that widely used block-diagonal curvature approximations can fail catastrophically, exhibiting order-one relative error and poor directional alignment even in mid-scale LLMs. Together, our results demonstrate that foundation-model Hessian spectra are both computable and qualitatively misrepresented by prevailing approximations, opening the door to principled curvature-based analysis at scale.
	\end{abstract}
	
	
	\section{Introduction}
	Second-order information, encapsulated by the Hessian of the training objective is central to understanding and improving modern deep learning systems. Curvature governs optimization dynamics and the efficacy of preconditioning \citep{martens2010hf,schraudolph2002fastcurv}, reveals the sharpness or flatness of minima associated with generalization \citep{keskar2017largebatch,li2018visualizing,foret2021sam}, and underpins data attribution tools such as influence functions that require inverse Hessian vector products \citep{koh2017influence}. Classically, Hessian vector products (HVPs) provide an efficient, matrix-free primitive for accessing curvature at the cost of a few backpropagations \citep{pearlmutter1994hvp}. 
	
	At \emph{foundation model} scale, however, computing HVPs becomes a systems problem. State-of-the-art training stacks shard parameters, optimizer states, and activations across devices using data/model parallelism and sharded optimizers \citep{shoeybi2019megatron,rajbhandari2020zero,zhao2023fsdp,narayanan2021megatronlm}. Naïvely applying standard HVP routines in these settings typically forces expensive full-parameter gathers, defeating the purpose of sharding and bloating memory and communication. As a result, practical curvature diagnostics, spectral tools, and inverse–Hessian solvers are rarely available during the development and operation of large LMs and vision models, despite their potential to guide training, monitoring, and deployment.
	
	\textbf{This paper.} We present a scalable framework for \emph{distributed} HVPs that is compatible with Fully Sharded Data Parallelism (FSDP). Our approach preserves model-parallel parameter partitioning to compute HVPs \emph{without} full gathers, enabling second-order analyses at (sparse and dense) trillion-parameter scale. The resulting primitive unlocks several downstream uses: (i) stochastic Lanczos methods to estimate spectral densities and traces \citep{hutchinson1989,ubaru2017slq,chen2021slq}, (ii) conjugate-gradient (CG) solvers for implicit second-order updates \citep{hestenes1952cg,martens2010hf,nocedal2006numopt}, and (iii) scalable influence-function pipelines for data sensitivity \citep{koh2017influence}. We demonstrate near-linear scaling of our distributed HVPs across language and vision models and show competitive wall-clock overheads relative to first-order training passes.
	
	\paragraph{Scope.}
	Our work targets \emph{foundation-model-scale} training and evaluation regimes that already rely on sharding (e.g., FSDP/ZeRO, tensor/pipeline parallelism) \citep{rajbhandari2020zero,shoeybi2019megatron,zhao2023fsdp,narayanan2021megatronlm}. We intentionally treat the HVP as a systems primitive, focusing on communication/memory efficiency and on composing the primitive with widely used spectral and inverse–Hessian routines.
	
	\vspace{-0.5em}
	\section{Related Work}
	
	\paragraph{Hessian vector products and second-order optimization.}
	Fast HVPs via reverse-over-forward automatic differentiation were introduced by \citet{pearlmutter1994hvp}. They enable Krylov methods for Newton/Gauss–Newton steps (e.g., CG) without forming the Hessian \citep{hestenes1952cg,nocedal2006numopt}. In deep learning, Hessian-free optimization (HF) operationalized these ideas with damping and curvature-vector products \citep{martens2010hf,martens2011rnn_hf}, while \citet{schraudolph2002fastcurv} developed efficient curvature–vector products for Fisher/Gauss–Newton approximations.
	
	\paragraph{Curvature, spectra, and generalization.}
	A rich line of work uses Hessian spectra to probe loss landscape geometry, training dynamics, and generalization \citep{sagun2017hessian,ghorbani2019hessian,li2018visualizing, granziol2019learningrate, granziol2021spectrum}. Stochastic trace and spectral estimators such as Hutchinson’s method and stochastic Lanczos quadrature (SLQ) make full-spectrum probing feasible using only matvecs \citep{hutchinson1989,ubaru2017slq,chen2021slq}. Practical toolkits such as PyHessian leverage HVPs for moderately sized models \citep{yao2019pyhessian}. Recent work reports power-law Hessian spectra but is limited to sub-billion-parameter models \citep{tang2025overlooked}.
	
	\paragraph{Influence functions and iHVPs.}
	Influence functions attribute predictions to training examples by computing inverse–Hessian vector products (iHVPs) \citep{koh2017influence}. Alternative estimators such as TracIn avoid explicit Hessians \citep{pruthi2020tracin}. Our work enables exact iHVPs at foundation scale.
	
	\paragraph{Distributed training at scale.}
	Modern foundation models rely on tensor, pipeline, and sharded data parallelism \citep{shoeybi2019megatron,rajbhandari2020zero,zhao2023fsdp,narayanan2021megatronlm,fedus2022switch,chowdhery2022palm}. We target this regime and make HVPs communication-aware.

\section{Curvature in LLMs: Progress, Motivation, and Contributions}
\label{sec:curvature-llms}

\paragraph{Why curvature at foundation scale?}
Second-order information provides capabilities absent from first-order pipelines:
(i) \emph{Monitoring \& reliability} via spectral norms and densities to flag instabilities, collapse, or domain shift;
(ii) \emph{Optimizer design} through conjugate-gradient implicit updates, damping, or curvature-conditioned schedules \citep{martens2010hf,nocedal2006numopt};
(iii) \emph{Data governance} through influence analyses that identify harmful or mislabeled data at scale \citep{koh2017influence}.
Making these tools routine in trillion-parameter regimes promises both efficiency and accountability.

\paragraph{Recent progress in LLM-scale curvature.}
Research has advanced on three complementary fronts. \textbf{Influence functions via curvature approximations.}
	\citet{grosse2023llmif} scale influence functions using EK-FAC (a K-FAC variant), avoiding explicit inverse–Hessian vector products. While tractable at tens of billions of parameters, this sacrifices exactness and relies on heuristics such as TF-IDF pre-filtering \citep{anthropicBlogIF}. \textbf{Foundation-scale spectra and tooling.}
	\citet{granziol2025hessformer} introduce HessFormer, performing Hessian spectral density estimation on 10-70B models via stochastic Lanczos quadrature. The implementation is single-node and limited in FLOP utilisation, restricting scalability. \textbf{Hessian structure and geometry.}
	\citet{tang2025overlooked} report power-law structure in Hessian spectra across CNNs and LLMs, with predictive power for generalization. Complementary work \citep{disipio2025info} frames LLM optimization through Fisher information geometry.

\paragraph{Systems challenge.}
All of these efforts highlight the importance of curvature but sidestep a key barrier: under FSDP/ZeRO, parameters are sharded, and naïve HVPs force expensive all-gathers. A practical second-order stack must preserve sharding, localize matvecs, and align communication with existing collectives.

\paragraph{Our contributions.}
\begin{itemize}
	\item \textbf{Shard-preserving HVP primitive.}
	We design Hessian vector products that operate directly on FSDP-sharded parameters, avoiding all-gathers while remaining autograd-compatible.
	\item \textbf{Scalable spectral and inverse–Hessian solvers.}
	We integrate this primitive with stochastic Lanczos quadrature \citep{hutchinson1989,ubaru2017slq,chen2021slq} and CG-based inverse-HVP solvers.
	\item \textbf{Empirical scaling.}
	We demonstrate near-linear strong scaling with node count and modest overhead relative to first-order passes.
\end{itemize}

\paragraph{Comparison with recent work.}
Table~\ref{tab:curvature_comparison} contrasts recent approaches with our shard-preserving HVP framework.

\begin{table*}[t]
	\centering
	\caption{Curvature/Hessian methods at LLM scale. ASDL and K-FAC-based systems (PipeFisher, KAISA) target curvature approximations rather than the true Hessian and are not FSDP-native. HessFormer does not scale across GPUs in this regime. In contrast, our primitive preserves FSDP sharding and achieves near-linear scaling.}
	\label{tab:curvature_comparison}
	\begin{tabular}{lllll}
		\toprule
		\textbf{Work} & \textbf{Method} & \textbf{Scale} & \textbf{Faithfulness} \\
		\midrule
		\citet{grosse2023llmif,anthropicBlogIF} & IF via EK-FAC (K-FAC) & $10^9$-$10^{10}$ & Approx. \\
		\citet{granziol2025hessformer} & SLQ + HVP (single node) & $10^{10}$-$7{\times}10^{10}$ & Exact \\
		\citet{tang2025overlooked} & Empirical Hessian spectra & CNNs, $10^8$-$10^9$ & Exact (analysis) \\
		\citet{disipio2025info} & Fisher/quantum geometry & Conceptual & N/A \\
		\textbf{Ours} & Shard-preserving HVP (FSDP) & $10^{12}$+ (multi-node) & Exact \\
		\bottomrule
	\end{tabular}
\end{table*}

\subsection{Shard-local Finite-Difference HvPs under FSDP}

We compute dataset-averaged Hessian vector products (HvPs) for FSDP models by a shard-local central finite-difference operator,
\begin{equation}
	\label{eq:grad-fd-hvp}
	Hv \;\approx\;
	\frac{\nabla_\theta L(\theta+\varepsilon v) - \nabla_\theta L(\theta-\varepsilon v)}{2\varepsilon},
\end{equation}
which discretises the continuous operator
\(
Hv = \nabla_\theta(D_v L(\theta))
\),
with \(D_v L(\theta) = \langle \nabla_\theta L(\theta), v \rangle\).
The perturbations in \eqref{eq:grad-fd-hvp} are applied shard-locally to FSDP parameters, avoiding parameter gathers and preserving data parallelism.

\begin{theorem}[Error of gradient finite-difference HvPs]
	\label{thm:grad-fd-hvp}
	Let $L \in C^4(\mathbb{R}^n)$, let $\theta \in \mathbb{R}^n$, and let $v \in \mathbb{R}^n$ be a unit vector. Consider the gradient finite-difference estimator defined in \eqref{eq:grad-fd-hvp}, computed in floating-point arithmetic with machine precision $\varepsilon_{\mathrm{mach}}$.
	
	Then the approximation error satisfies
	\begin{equation}
		\|\widetilde{H}v - Hv\|
		\;\le\;
		\frac{\varepsilon^2}{6}\,
		\|\nabla_\theta(D_v^3 L(\theta))\|
		+
		O\!\left(
		\frac{\varepsilon_{\mathrm{mach}}}{\varepsilon}
		\|\nabla_\theta L(\theta)\|
		\right).
	\end{equation}
	Moreover, the error is minimised for
	\begin{equation}
		\varepsilon^\star
		\;\asymp\;
		\left(
		\frac{\varepsilon_{\mathrm{mach}}\,
			\|\nabla_\theta L(\theta)\|}
		{\|\nabla_\theta(D_v^3 L(\theta))\|}
		\right)^{1/3},
	\end{equation}
	at which the minimal achievable error scales as
	\begin{equation}
		\|\widetilde{H}v - Hv\|
		=
		O\!\left(
		\varepsilon_{\mathrm{mach}}^{2/3}
		\|\nabla_\theta L(\theta)\|^{2/3}
		\|\nabla_\theta(D_v^3 L(\theta))\|^{1/3}
		\right).
	\end{equation}
\end{theorem}

\begin{remark}[Practical step sizes]
	For float32 arithmetic, $\varepsilon_{\mathrm{mach}} \approx 1.2\times 10^{-7}$,
	yielding $\varepsilon^\star \sim 10^{-3}$  $10^{-2}$ and typical HvP errors of order
	$10^{-5}$. For bfloat16, $\varepsilon_{\mathrm{mach}} \approx 3.9\times 10^{-3}$,
	giving $\varepsilon^\star \sim 10^{-1}$ and HvP errors of order $10^{-2}$. These
	values are consistent with empirical stability observed in large-scale FSDP
	experiments.
\end{remark}

\begin{algorithm}[H]
	\caption{Shard-local finite-difference HvP under FSDP}
	\label{alg:fsdp-hvp}
	\begin{algorithmic}[1]
		\STATE \textbf{Input:} parameters $\theta$, direction $v$, step $\varepsilon$, dataloader
		\STATE \textbf{Output:} shard-local Hessian vector product $Hv$
		\STATE Perturb local DTensor shards: $\theta \leftarrow \theta + \varepsilon v$
		\STATE Zero gradients
		\FOR{batch $b$ in dataloader}
		\STATE Compute loss $\ell_b(\theta + \varepsilon v)$
		\STATE Accumulate $\nabla_\theta \ell_b \cdot w_b$ into \texttt{param.grad}
		\ENDFOR
		\STATE Perturb local shards: $\theta \leftarrow \theta - 2\varepsilon v$ \hfill // now at $\theta - \varepsilon v$
		\FOR{batch $b$ in dataloader}
		\STATE Compute loss $\ell_b(\theta - \varepsilon v)$
		\STATE Accumulate $-\nabla_\theta \ell_b \cdot w_b$ into \texttt{param.grad}
		\ENDFOR
		\STATE Restore parameters: $\theta \leftarrow \theta + \varepsilon v$
		\STATE Collect shard-local gradients into FP32 buffers, set \texttt{param.grad = None}
		\STATE Scale by $2\varepsilon \sum_b w_b$ to obtain local HvP result
		\STATE \textbf{return} shard-local tensors representing $Hv$
	\end{algorithmic}
\end{algorithm}

\paragraph{FSDP-native HVP primitive.}
Algorithm~\ref{alg:fsdp-hvp} implements a dataset-averaged finite-difference
Hessian  vector product under FSDP ZeRO-3. Parameter perturbations
$\theta \mapsto \theta \pm \varepsilon v$ are applied \emph{in place} on each
rank’s local DTensor shard, avoiding any parameter gathers.
Each HVP requires two standard FSDP gradient passes over the same dataloader
slice, and three shard-local AXPY updates to restore parameters.
The resulting gradients are accumulated locally and scaled by $(2\varepsilon)^{-1}$,
yielding shard-local pieces of $Hv$.
Beyond the communication already incurred by FSDP during backpropagation,
the only additional collective is at most a scalar all-reduce to normalize $v$
or form global dot-products in Krylov routines.

\subsection{Finite Difference as Local Averaging}

For a function $f : \mathbb{R}^n \to \mathbb{R}$, the central finite-difference
approximation to the directional second derivative along a unit vector $v$ is
\begin{align}
	H_\varepsilon(x)[v]
	&=
	\frac{
		f(x + \varepsilon v)
		-
		2 f(x)
		+
		f(x - \varepsilon v)
	}{
		\varepsilon^2
	}.
\end{align}

This approximation can be written exactly as a weighted average of the true
Hessian along the line $x + t v$.
\begin{figure*}[t]
	\centering
	\begin{subfigure}[t]{0.24\textwidth}
		\centering
		\includegraphics[width=\textwidth]{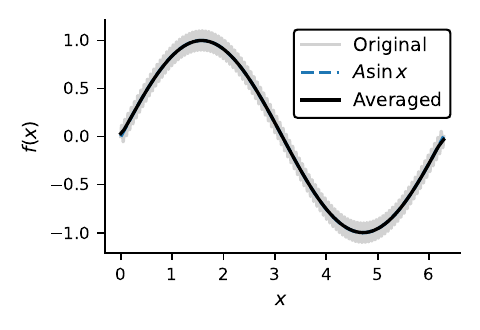}
		\caption{Local averaging in 1D.}
		\label{fig:avg1d}
	\end{subfigure}
	\hfill
	\begin{subfigure}[t]{0.24\textwidth}
		\centering
		\includegraphics[width=\textwidth]{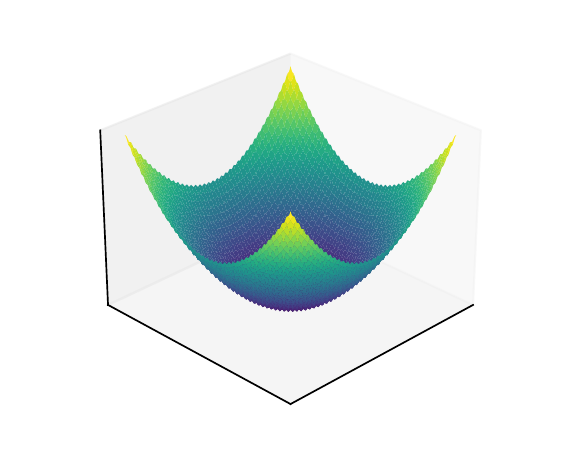}
		\caption{Convex rippled surface.}
		\label{fig:surface2d}
	\end{subfigure}
	\hfill
	\begin{subfigure}[t]{0.24\textwidth}
		\centering
		\includegraphics[width=\textwidth]{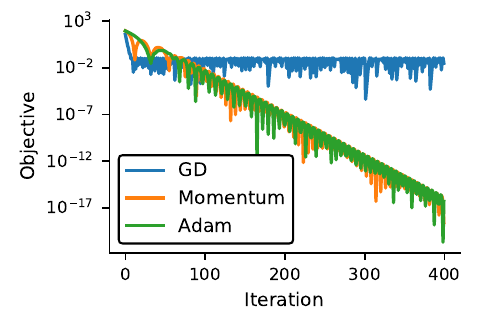}
		\caption{Optimiser convergence.}
		\label{fig:bestlr}
	\end{subfigure}
	\hfill
	\begin{subfigure}[t]{0.24\textwidth}
		\centering
		\includegraphics[width=\textwidth]{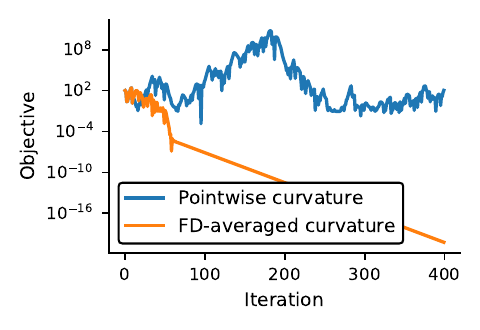}
		\caption{Adaptive Nesterov}
		\label{fig:nesterov}
	\end{subfigure}
	
	\caption{
		Local averaging suppresses high-frequency curvature without altering large-scale geometry.
		(a) Finite differences act as a local averaging operator on oscillatory signals.
		(b) A globally convex surface with high-frequency curvature.
		(c) Best achievable convergence under grid-searched learning rates for GD, momentum, and Adam for surface (b).
		(d) Adaptive Nesterov acceleration using pointwise versus finite-difference averaged curvature on surface (b).
	}
	\label{fig:local_averaging_optimization}
\end{figure*}

\begin{theorem}
	For $f \in C^2(\mathbb{R}^n)$,
	\begin{align}
		H_\varepsilon(x)[v]
		&=
		\frac{1}{\varepsilon}
		\int_{-\varepsilon}^{\varepsilon}
		w_\varepsilon(t)\;
		v^\top H(x + t v)\, v
		\; dt,
	\end{align}
	where
	\begin{align}
		w_\varepsilon(t)
		&=
		1
		-
		\frac{|t|}{\varepsilon},
	\end{align}
	and $H(x) = \nabla^2 f(x)$.
\end{theorem}
See App.~\ref{sec:localavg} for proof. The weighting function $w_\varepsilon$ is nonnegative, symmetric, and supported
on $[-\varepsilon,\varepsilon]$, with
$
	\frac{1}{\varepsilon}
	\int_{-\varepsilon}^{\varepsilon}
	w_\varepsilon(t)\,dt
	= 1.
$
It attains its maximum at $t=0$ and decays linearly to zero at the boundary,
corresponding to a triangular averaging kernel. 

As shown in Figure \ref{fig:avg1d}  the function \[
f(x) = A \sin x + B \sin(\omega x),
\] or (Figure \ref{fig:surface2d}) its higher dimensional equivalent 
\[
f(x,y) = \frac{1}{2}(x^2 + y^2) + B \sin(\omega x)\sin(\omega y),
\]
where $w\gg 1$. Such functions locally have large curvature but are globally smooth. Such a surface exhibits known tendencies in deep learning such as the improved performance of momentum SGD and Adam over SGD (Figure \ref{fig:bestlr}). We find that for this toy problem, setting the per iteration learning rates and Nesterov \citep{nesterov1983method} momentum
\[
\alpha = \frac{1}{L},
\qquad
\beta = \frac{\sqrt{L} - \sqrt{m}}{\sqrt{L} + \sqrt{m}}.
\]
We find that only pointwise curvature (Figure \ref{fig:nesterov}) converges.
\section{(Stochastic) Lanczos (Quadrature)}
\label{sec:lanczos}
The Lanczos algorithm approximates eigenvalues of a symmetric matrix via the recurrence
\begin{equation}
	\beta_{k+1} q_{k+1} = H q_k - \alpha_k q_k - \beta_k q_{k-1}, \thinspace
	\alpha_k = q_k^\top H q_k.
\end{equation}
After $m$ steps, $H$ is approximated by a tridiagonal $T_m$, whose Ritz values approximate extremal eigenvalues. In finite precision, loss of orthogonality yields ghost eigenvalues \citep{paige1971}. With random probes, Lanczos enables stochastic Lanczos quadrature (SLQ) \citep{ubaru2017}, estimating
\begin{equation}
	v^\top f(H) v \;\approx\; e_1^\top f(T_m) e_1,
\end{equation}
recovering the spectral density via polynomial moments encoded by the Krylov process. Despite extensive use at moderate scale, SLQ had not been demonstrated at LLM scale due to the cost of repeated HvPs. Our shard-local HvP enables SLQ under FSDP. We quantify how the Lanczos algorithm adapts under finite difference Hessian vector products with the following theorem.
\begin{theorem}[Error of Lanczos with finite-difference Hessian  vector products]
	\label{theorem:lanczosfindif}
	Let $H = \nabla^2 f(x) \in \mathbb{R}^{P \times P}$ be symmetric, and let $\tilde T_m$ be the tridiagonal matrix produced by $m$ steps of the symmetric Lanczos algorithm applied to $H$, where Hessian  vector products are approximated using the central finite-difference estimator of Theorem~\ref{thm:grad-fd-hvp} with step size $\varepsilon^\star$.
	
	Assume the finite-difference errors are unbiased and independent across Lanczos iterations. Then the computed tridiagonal satisfies
	\[
	\tilde T_m = T_m + \Delta T_m,
	\]
	where the perturbation obeys
	\[
	\mathbb{E}\|\Delta T_m\|_2
	\;\lesssim\;
	\|H\|_2\,
	\sigma_f^{1/2}\,\|D^4 f\|^{1/2}
	\sqrt{\frac{m\,\bar\eta}{P}}.
	\]
	
	Consequently, the Ritz values $\tilde\lambda_i$ satisfy
	\[
	|\tilde\lambda_i - \lambda_i|
	\;\lesssim\;
	\|H\|_2\,
	\sigma_f^{1/2}\,\|D^4 f\|^{1/2}
	\sqrt{\frac{m\,\bar\eta}{P}}
	\quad\text{in expectation}.
	\]
	
	Moreover, loss of orthogonality among Lanczos vectors induces spurious duplicate (“ghost”) eigenvalues with RMS splitting
	\[
	\Delta\lambda_{\mathrm{ghost}}
	\;\sim\;
	2\,\|H\|_2\,
	\sigma_f^{1/2}\,\|D^4 f\|^{1/2}
	\sqrt{\frac{k\,\bar\eta}{P}},
	\]
	where $k$ denotes the number of Lanczos iterations after convergence of the associated Ritz vector.
\end{theorem}

\begin{remark}
	When the finite-difference step size is chosen optimally as in Theorem~\ref{thm:grad-fd-hvp}, we reduce to Paiges classic result.
	The quantity
	\[
	\bar\eta
	\;\equiv\;
	\bigl(\mathbb{E}\|q\|_\infty^2\bigr)^{-1},
	\]
	where $q$ denotes a typical Lanczos vector, measures the effective delocalisation of Krylov basis vectors. 
\end{remark}

\begin{table*}[!h]
	\centering
	\small
	\begin{tabular}{l r r l l r r r}
		\hline
		Model & Batch & Steps & Reorth & Subsample & $T_{\mathrm{HvP}}$ (s) & $T_{\mathrm{Lanczos}}$ (s) & Overhead \\
		\hline
		DeepSeek-R1-Distill-Qwen-7B & 1 & 1 & False & 1/1000 & 23.99 & 24.30 & $\approx$1.3\% \\
		Qwen-1.5B                  & 2 & 8 & True  & 1/200  & 71.26 & 72.53 & $\approx$1.8\% \\
		\hline
	\end{tabular}
	\vspace{3pt}
	\caption{Lanczos overhead relative to the HVP primitive.}
	\label{tab:lanczos-overhead}
\end{table*}

\section{Systems Characterization: Cost, Scaling, and Baselines}
\label{sec:cost-model}
Let $T_{\mathrm{grad}}$ denote the wall-clock time for a single distributed FSDP gradient evaluation: one forward and one backward pass over the dataloader slice, including FSDP parameter all-gathers and gradient reduce-scatter/all-reduce collectives, but no optimizer update.
In a standard $(\alpha,\beta,\gamma)$ model, per rank:
\begin{equation*}
	\label{eq:tgrad}
	T_{\mathrm{grad}}
	\;=\;
	F_{\mathrm{fwd}}\,\gamma
	\;+\;
	F_{\mathrm{bwd}}\,\gamma
	\;+\;
	\alpha\, G_{\mathrm{grad}}
	\;+\;
	\beta\, V_{\mathrm{grad}},
\end{equation*}
where $F_{\mathrm{fwd}},F_{\mathrm{bwd}}$ are forward/backward FLOPs, $G_{\mathrm{grad}}$ is the number of collectives invoked by FSDP, and $V_{\mathrm{grad}}$ is the total communicated payload per rank.

\paragraph{Cost of one dataset-averaged HVP.}
All additional work outside the two FSDP passes consists of linear-time vector operations over $P_{\mathrm{loc}}$ parameters (AXPYs, copies, scaling) plus at most one scalar all-reduce. Grouping these into $T_{\mathrm{vec}}$ yields:
\begin{equation*}
	\label{eq:thvp}
	T_{\mathrm{HvP}}
	\;=\;
	2\,T_{\mathrm{grad}}
	\;+\;
	T_{\mathrm{vec}},
	\thinspace
	T_{\mathrm{vec}} = \mathcal{O}(P_{\mathrm{loc}}\,\gamma) + \mathcal{O}(T_{\mathrm{scalar}}),
\end{equation*}
where $T_{\mathrm{scalar}}$ denotes the cost of a scalar all-reduce.
Importantly, the HVP has \emph{exactly} the same FSDP communication pattern as two gradient evaluations, and introduces no additional $O(P)$-sized collectives (no full-parameter all-gathers beyond those already implied by FSDP execution).

\paragraph{Cost of one Lanczos iteration.}
Each Lanczos step requires one Hessian  vector product, a constant number of
global scalar reductions to form recurrence coefficients, and shard-local
vector operations.
With sliding-window reorthogonalisation of size $r$, this incurs at most
$(2+r)$ scalar all-reduces and $O((1+r)P_{\mathrm{loc}})$ local AXPY operations.
Thus one iteration costs
\begin{equation*}
	T_{\mathrm{Lanczos\ iter}}(r)
	\;=\;
	T_{\mathrm{HvP}}
	\;+\;
	(2+r)\,T_{\mathrm{scalar}}
	\;+\;
	O\!\left((1+r)\frac{P}{R}\gamma\right).
\end{equation*}

\paragraph{End-to-end SLQ/CG complexity (for context).}
For stochastic Lanczos quadrature (SLQ) with $m$ Krylov steps and $s$ random probe vectors, the total time is
\begin{equation*}
	T_{\mathrm{SLQ}} \approx s\,m\,T_{\mathrm{Lanczos\ iter}}(r) + T_{\mathrm{post}},
\end{equation*}
where $T_{\mathrm{post}}$ (tridiagonal eigendecompositions and quadrature) is typically negligible compared to the $s m$ HVP calls at scale. A CG-based inverse-HVP solver similarly reduces to repeated HVP calls plus scalar reductions and shard-local vector operations, and therefore inherits the same cost structure.

\subsection{Empirical scaling}
Our HVP inherits the asymptotic scaling of a standard FSDP gradient evaluation. In particular, when a model \emph{fits} on a single rank, FSDP ZeRO-3 can be strictly slower than standard data parallelism (DP) because ZeRO-3 trades memory savings for more frequent communication. We formalise this in Appendix \ref{sec:fsdpvsdp}. We display single node scaling results in Table \ref{tab:strong-scaling-single-node}
\begin{table}[!h]
	\centering
	\small
	\begin{tabular}{r r r r r }
		\hline
		\#GPUs & Avg t &  Std & Peak mem & Speedup \\
		\hline
		1 & 12.40s & 0.05s & 67.63GB &  1.00$\times$ \\
		2 & 7.93s  & 0.04s & 68.74GB &  1.56$\times$ \\
		4 & 4.06s  & 0.10s & 67.22GB &  3.05$\times$ \\
		8 & 2.16s  & 0.03s & 66.51GB &  5.74$\times$ \\
		\hline
	\end{tabular}
	\vspace{3pt}
	\caption{Strong scaling of the HVP primitive on one node. Batch size $160$ on Qwen $0.6$B, sequence length $1024$ on wikitext-$2$.}
	\vspace{-15pt}
	\label{tab:strong-scaling-single-node}
\end{table}
and multi node scaling results in Table \ref{tab:multi-node-scaling}.
\begin{table}[!h]
	\centering
	\small
	\begin{tabular}{r r r r r r}
		\hline
		\#Nodes & B & Avg t & Std & Peak mem & Speedup \\
		\hline
		1  & 64 & 9.36s & 0.07s & 62.70GB & 1.00$\times$ \\
		2 & 72 & 5.20s & 0.06s & 67.31GB & 1.80$\times$ \\
		\hline
	\end{tabular}
	\vspace{3pt}
	\caption{Multi-node scaling of the HVP primitive. $B$ denotes batch size and GPU count is $8 \times$ Nodes. Qwen $32$B on wikitext-$2$ sequence length $1024$.}
	\label{tab:multi-node-scaling}
\end{table}
As shown in Table \ref{tab:hessformer-compare}, for the one available multi-GPU hessian calculation method \citep{granziol2025hessformer} we see significant ($4$x) speedups, justifying our contribution at a cost per experiment level.
\begin{table}[!h]
	\centering
	\small
	\begin{tabular}{l r r r r}
		\hline
		Method & Batch & GPUs & Avg time & Peak mem\\
		\hline
		HessFormer & 32 & 8 & 89.57s & 71.91GB \\
		FSHP      & 80 & 8 & 22.21s & 74.71GB \\
		\hline
	\end{tabular}
	\vspace{3pt}
	\caption{Single-node comparison on a 32B Qwen model.}
	\label{tab:hessformer-compare}
\end{table}
\paragraph{Empirical Lanczos overhead}Lanczos adds only shard-local vector operations and scalar reductions on top of two gradient passes. This is a small overall overhead relative to the hessian vector product in typical setups, we confirm this experimentally in Table \ref{tab:lanczos-overhead}.
\section{Exp I: Block-diagonal Hypothesis Test}
\label{sec:block-diagonal}

To illustrate analyses enabled by our primitive, we test the common assumption that the Hessian is approximately block diagonal. In optimisation, natural-gradient methods replace the intractable full Fisher Information Matrix with explicit layer-wise block-diagonal approximations, most notably through Kronecker-factored approaches such as K-FAC \citep{amari1998natural,martens2015kfac}. In pruning, while early second-order methods assumed a diagonal Hessian \citep{lecun1990optimal}, subsequent work showed that full-Hessian formulations are computationally infeasible beyond small models \citep{hassibi1993optimal}, motivating block-wise and layer-wise approximations in practice. In large-scale post-training quantisation, block-diagonal structure is enforced by design: methods such as GPTQ optimise second-order objectives independently over small parameter blocks using local Hessian approximations, discarding cross-block interactions entirely \citep{frantar2022gptq}.

To evaluate the validity of a block-diagonal approximation of the Hessian in transformer layer space, we test whether cross-block curvature terms are negligible in practice. Let $\theta \in \mathbb{R}^P$ denote the flattened parameter vector, partitioned into contiguous blocks $\{\theta^{(b)}\}$ corresponding to individual transformer layers. For a randomly sampled probe vector $v \sim \mathcal{N}(0, I)$ with $\|v\|_2 = 1$, we compute the full Hessian  vector product
\[
h_{\mathrm{full}} = H v,
\]
using a finite-difference estimator over minibatches. For each block $b$, we construct a masked probe $v^{(b)}$ by zeroing all components of $v$ outside block $b$, and compute the corresponding masked product $h_{\mathrm{masked}}^{(b)} = H v^{(b)}$.

We then restrict both products to the same block and compare $(H v)^{(b)}$ with $(H v^{(b)})^{(b)}$, which isolates the contribution of cross-block Hessian terms. Deviations are quantified via the relative error
\[
\frac{\| (H v)^{(b)} - (H v^{(b)})^{(b)} \|}
{\| (H v)^{(b)} \|},
\]
and cosine similarity between the two vectors. If the Hessian were approximately block-diagonal, these quantities would be close to zero and one, respectively.

\subsection{Model on generic Data}
For Qwen-0.6B on WikiText, parameters are split into 28 blocks. For a probe vector $v$, we compare $Hv$ to a block-diagonal proxy $H_{\mathrm{block}}v$, the results shown in Table \ref{tab:block-diagonal}. Here we see, that for a general LLM on a general dataset we have an $O(1)$ magnitude error and some but not overwhelming directional alignment as measured by cosine similarity.
\begin{table}[!h]
	\centering
	\small
	\begin{tabular}{l r}
		\hline
		Metric & Mean $\pm$ Std Dev \\
		\hline
		$\|H_{\mathrm{block}}v - Hv\|$ & $0.082 \pm 0.054$ \\
		Relative error & $0.941 \pm 0.059$ \\
		Cosine similarity & $0.311 \pm 0.120$ \\
		Blocks & 28 \\
		\hline
	\end{tabular}
	\vspace{3pt}
	\caption{Block-diagonal approximation error on Qwen-0.6B.}
	\label{tab:block-diagonal}
\end{table}
\subsection{Code generation fine-tuning experiment}
We fine-tune a causal language model on a coding dataset using a standard autoregressive objective, with loss applied only to the solution tokens. Each example concatenates a prompt and its corresponding solution, but prompt tokens are masked in the loss so that gradients arise solely from conditional solution generation. Writing $x = (x_{\mathrm{prompt}}, x_{\mathrm{sol}})$ for the full token sequence, the training objective is
\[
\mathcal{L}(\theta)
=
\sum_{t \in x_{\mathrm{sol}}}
\ell\big(f_\theta(x_{\le t}), x_{t+1}\big),
\]
which concentrates curvature on parameters involved in mapping prompts to solutions rather than prompt encoding alone. Sequences are truncated to a fixed maximum length and processed independently. Optimization is performed with AdamW under distributed training using \texttt{Accelerate}, with optional gradient checkpointing and mixed-precision arithmetic. The model remains in training mode throughout, ensuring that subsequent Hessian vector products correspond to the true training objective rather than an evaluation-time surrogate, and thus reflect the second-order structure encountered by the optimizer.

Empirically, as shown in Table \ref{tab:block-diagonal-summary} for Deepseek $1.3$Bn coder, in we observe large relative errors and near-zero mean cosine similarity across most blocks, indicating strong cross-layer coupling in the curvature. The only notable exception occurs in the final transformer block, where the cosine similarity is substantially higher, consistent with the reduced downstream mixing and MLP-dominated structure of the last layer.

\begin{table}[!h]
	\centering
	\small
	\begin{tabular}{l r}
		\hline
		Metric & Mean $\pm$ Std Dev \\
		\hline
		$\|H_{\mathrm{block}}v - Hv\|$ & $0.00442 \pm 0.00444$ \\
		Relative error & $1.214 \pm 0.150$ \\
		Cosine similarity & $0.00443 \pm 0.127$ \\
		Blocks & 24 \\
		\hline
	\end{tabular}
	\vspace{3pt}
	\caption{
		Block-diagonal approximation error (per-block statistics).
	}
	\label{tab:block-diagonal-summary}
\end{table}

\section{Exp II: Hessian Spectra}
 Following prior spectral studies \citep{ghorbani2019investigation,papyan2019measurements,ubaru2017}, we focus on estimating the spectral density of the Hessian rather than computing individual eigenvalues. To the best of our knowledge, Figure \ref{fig:gptoss120b} is the first spectral density estimate for causal language models at $100$Bn parameter scale.
 \begin{figure}[!h]
 	\centering
 	\includegraphics[width=\linewidth, trim=0 0 0 20, clip]{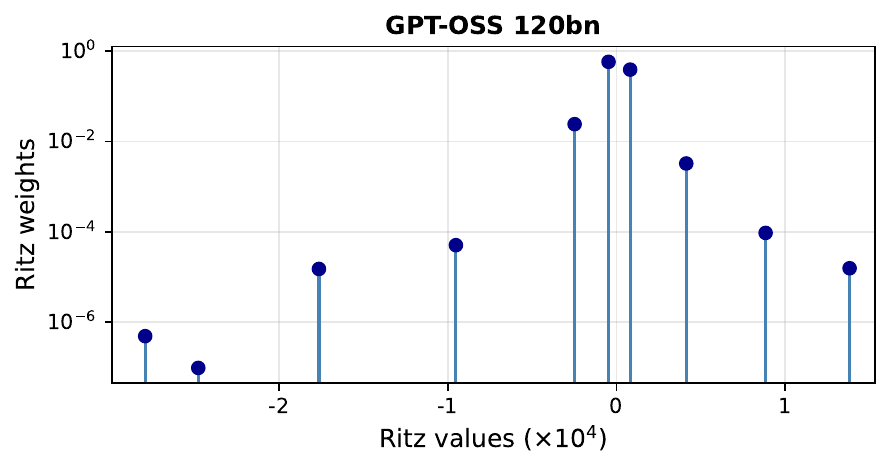}
 	\caption{Spectral Density of OpenAI $120$Bn parameter open weight model on wikitext-$2$, sequence length $64$, $\epsilon=10^{-4}$.}
 	\label{fig:gptoss120b}
 \end{figure}
The extremely large negative eigenvalues indicate that the data as tokenised has clearly not been memorised (there are directions of steep descent available). This may provide evidence to the problem of hallucination. Recollection of non memorised data will by definition be lossy. However since we do not know the details of OpenAIs $120$bn model \citep{openai2025gptoss120bgptoss20bmodel} training regimen and dataset, the plot is only indicative. We also do not have the resources to run several tests with varying dataset sizes to understand the magnitude of broadening \citep{granziol2019learningrate}. Do to the model size, only the essential Lanczos vectors were stored (not the entire krylov subspace) and no re-orthogonalisation was performed. Due to the small number of iterations ($10$) and using Theorems \ref{theorem:lanczosfindif} \& \ref{thm:grad-fd-hvp} we do not see the obvious arthefacts of low precision arithmetic distorting our computations, indicating correct analysis within the scale of computation.

\subsection{Lanczos vectors precision, $\epsilon$ and sub-sampling}
Given that our method requires storing the full Krylov subspace in memory\footnote{Future work could offload portions of the subspace to CPU or disk to reduce memory pressure and enable larger-scale runs.}, we study the numerical precision required to obtain stable Hessian spectral estimates, as well as the sensitivity of the spectrum to the smoothing parameter \(\epsilon\).

\begin{figure}[!h]
	\centering
	\includegraphics[width=\linewidth]{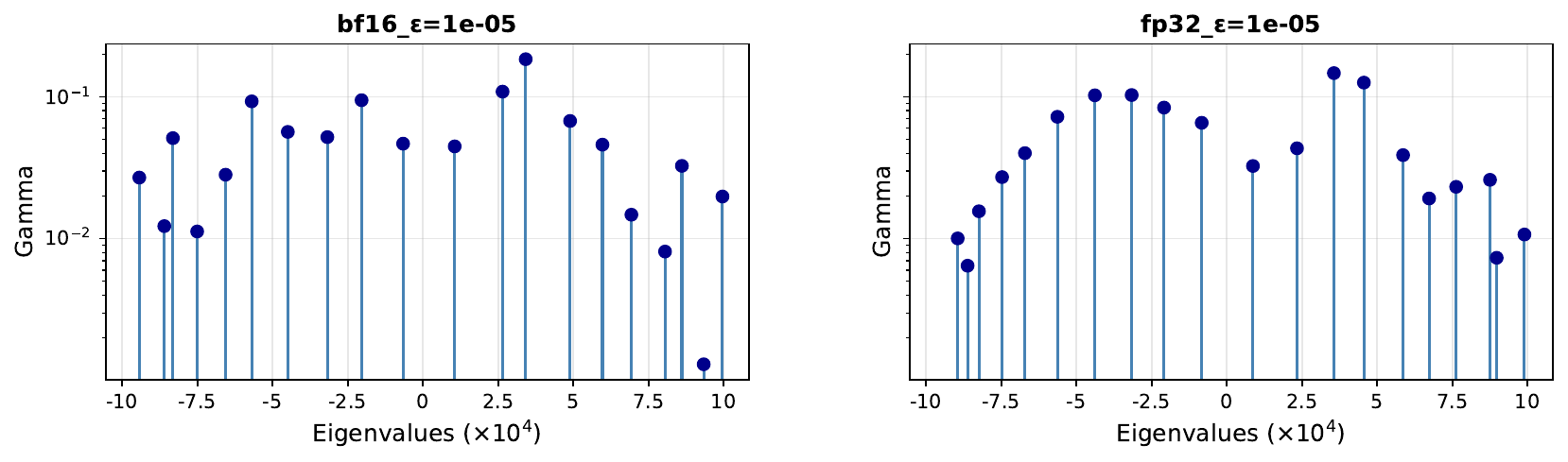}
	\caption{Pairwise structure for $\varepsilon = 10^{-5} \ll \varepsilon_{\mathrm{opt}}$. Matrix spectra resembles that of a random matrix with no clear outliar seperation.}
	\label{fig:pairs_eps_1em05}
\end{figure}

\begin{figure}[!h]
	\centering
	\includegraphics[width=\linewidth]{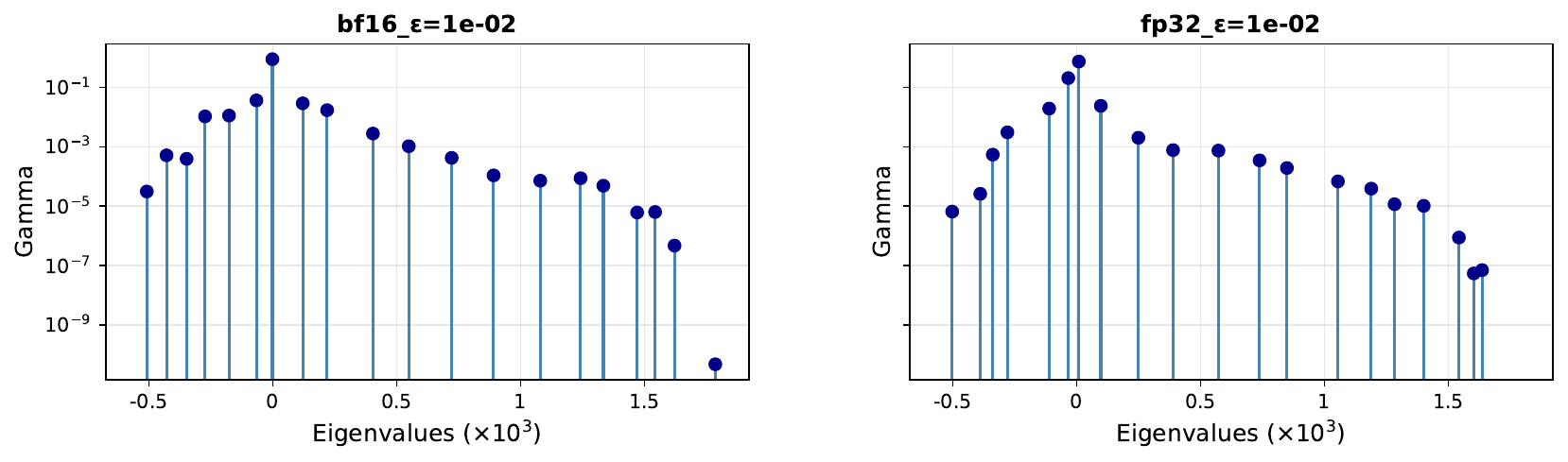}
	\caption{Pairwise structure for $\varepsilon = 10^{-2}  \gg \varepsilon_{\mathrm{opt}}$. Matrix spectra resembles that of a random matrix with no clear outliar seperatioon.}
	\label{fig:pairs_eps_1em02}
\end{figure}

We evaluate these effects using the Pythia-160M model trained on the Pile dataset, which we treat as an in-distribution setting for Hessian estimation. Experiments are conducted on a \(1\%\) subset of the Pile (1k examples drawn from a 100k-row deduplicated subset). All Lanczos scalar operations  -including inner products and entries of the tridiagonal \(T\) matrix  -are performed in float32, while only the Lanczos basis vectors are stored in either bfloat16 or float32. We perform a grid search over \(\epsilon\), apply full re-orthogonalisation at every Lanczos iteration, and retain the entire Krylov subspace.


We use a sequence length of 2048 tokens, padding each batch to the maximum sequence length (capped at 2048) using \texttt{pad\_token\_id} (or \texttt{eos\_token\_id} when unset). Labels are padded with \texttt{-100}, and attention masks assign ones to real tokens and zeros to padding.

Comparing bfloat16 and float32 Lanczos vectors on the same axes, we observe only minimal deviations between the resulting spectra, indicating that bfloat16 storage is sufficient for faithful spectral estimation. Clear Hessian structure, manifested as isolated spectral outliers, emerges for \(10^{-4} \le \epsilon \le 10^{-3}\). Around \(\epsilon \approx 10^{-3}\), modest differences appear: bfloat16 exhibits mild outlier broadening.


%

\begin{figure}[!h]
	\centering
	\includegraphics[width=\linewidth]{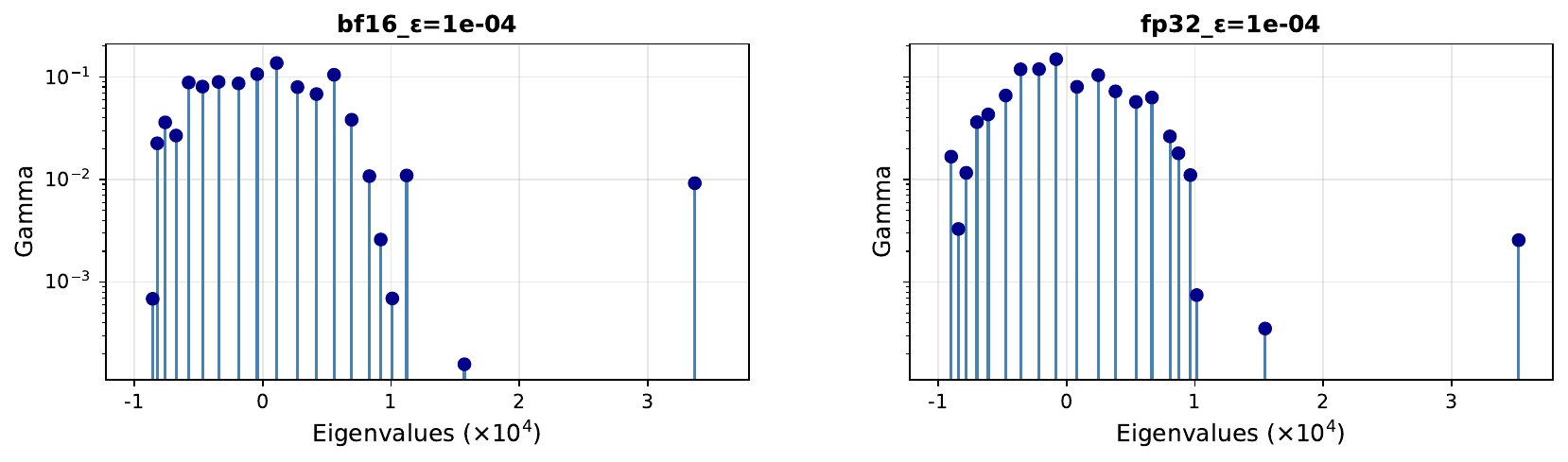}
	\caption{Pairwise structure for $\varepsilon = 10^{-4} \approx  \varepsilon_{\mathrm{opt}}$.}
	\label{fig:pairs_eps_1em04}
\end{figure}


\begin{figure}[!h]
	\centering
	\includegraphics[width=\linewidth]{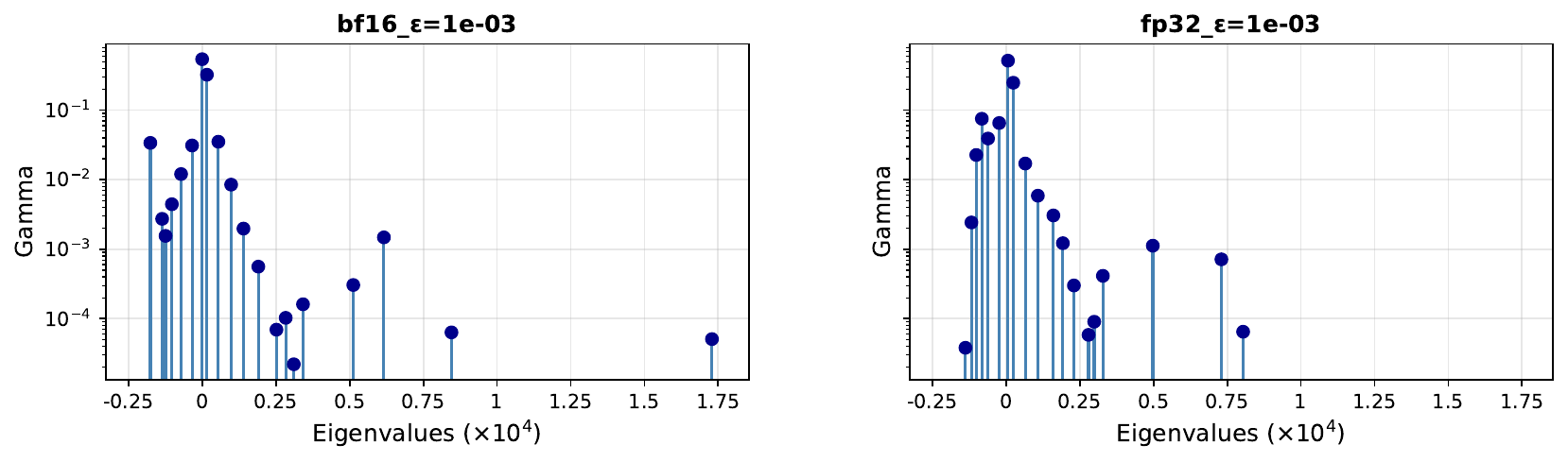}
	\caption{Pairwise structure for $\varepsilon = 10^{-3} \approx  \varepsilon_{\mathrm{opt}}$.}
	\label{fig:pairs_eps_1em03}
\end{figure}

\begin{figure}[!b]
	\centering
	\includegraphics[
	width=\linewidth,
	trim=0 0 0 20,
	clip
	]{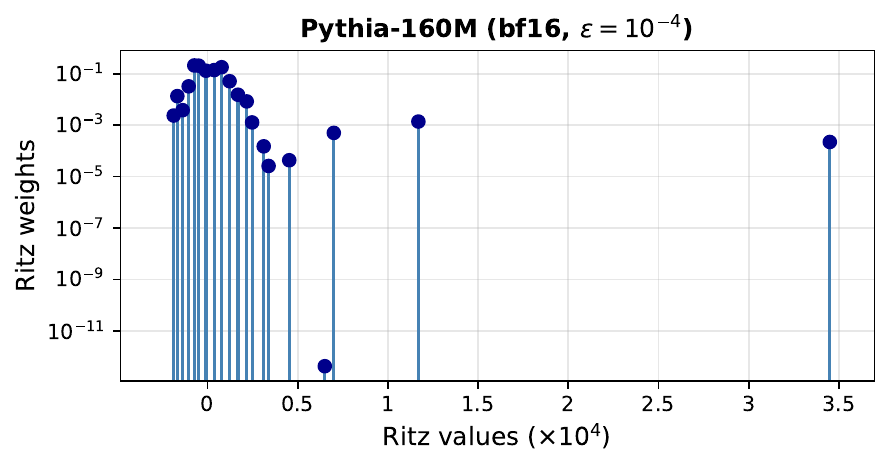}
	\caption{Results for \texttt{Pythia-160M} (bf16) at $\varepsilon = 10^{-4}$ using a 30\% subtraining split.}
	\label{fig:pythia160m_bf16_eps1e4_subtrain30pct}
\end{figure}

\begin{figure}[!t]
	\centering
	\begin{subfigure}{0.48\linewidth}
		\centering
		\includegraphics[width=\linewidth]{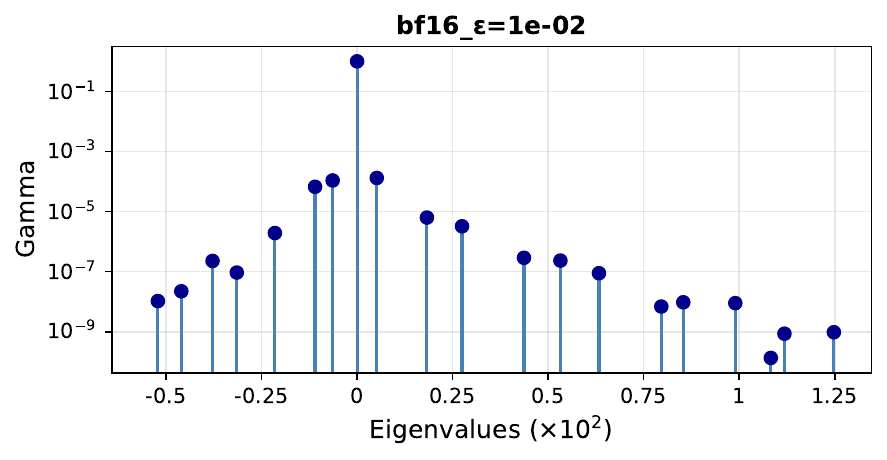}
		\caption{$\varepsilon = 10^{-2}$}
		\label{fig:stem_a}
	\end{subfigure}
	\begin{subfigure}{0.48\linewidth}
		\centering
		\includegraphics[width=\linewidth]{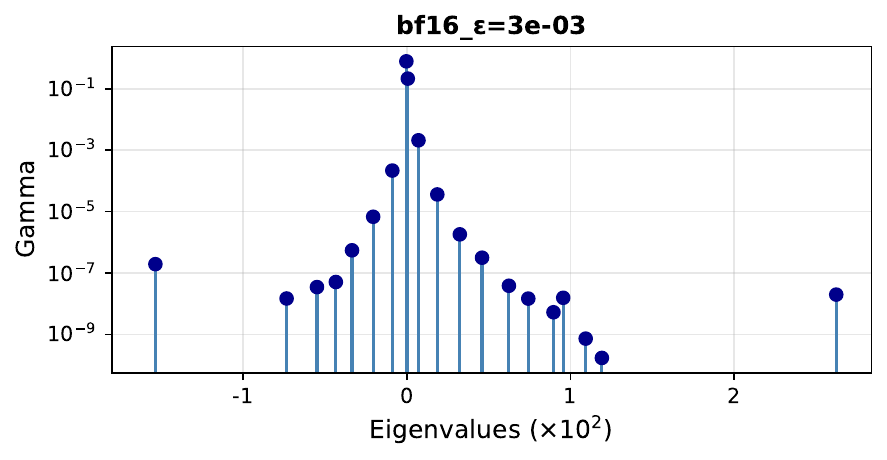}
		\caption{$\varepsilon = 3\times10^{-3}$}
		\label{fig:stem_b}
	\end{subfigure}
	\begin{subfigure}{0.48\linewidth}
		\centering
		\includegraphics[width=\linewidth]{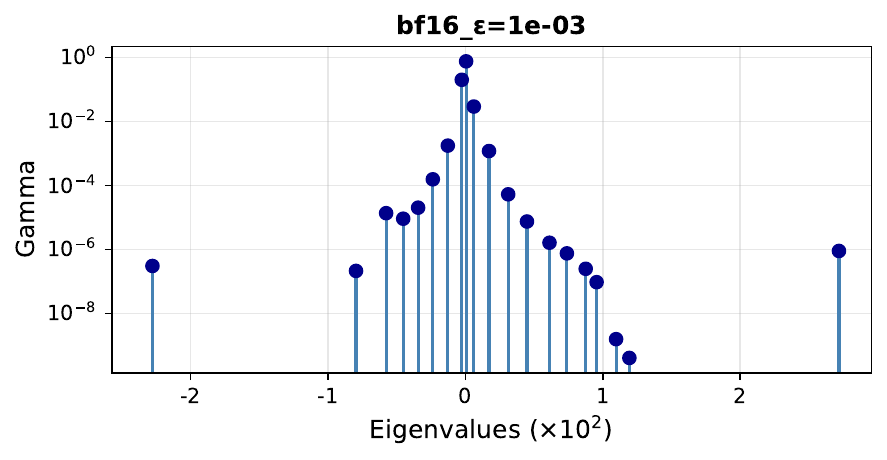}
		\caption{$\varepsilon = 10^{-3}$}
		\label{fig:stem_c}
	\end{subfigure}
	\begin{subfigure}{0.48\linewidth}
		\centering
		\includegraphics[width=\linewidth]{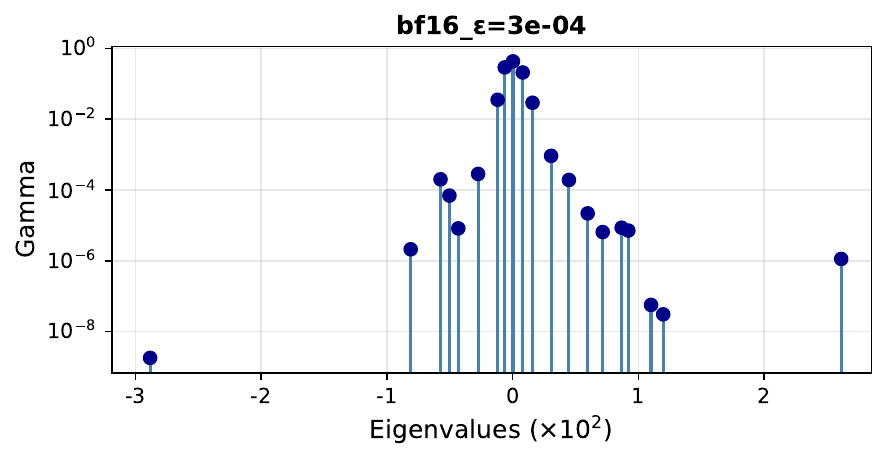}
		\caption{$\varepsilon = 3\times10^{-4}$}
		\label{fig:stem_d}
	\end{subfigure}
	\caption{Stem plots of $\gamma$ versus Hessian eigenvalues for \textbf{Pythia-1.4B (bf16)}.}
	\label{fig:stem_pair_2}
\end{figure}


\paragraph{Impact of Subsampling:} We find as shown in Figure \ref{fig:pythia160m_bf16_eps1e4_subtrain30pct} that increasing the in distribution dataset size by a factor of $30$ does not meaningfully alter the spectral distribution. This indicates using random matrix theretic arguments \citep{granziol2019learningrate} that the variance per Hessian element $H_{ij}$ is low. This implies that we expect curvature based methods such as second order or stochastic gradient optimisers to perform well at small batch size and that the curvature based quantisation methods require a small calibration set, both of which are observed in practice.

\paragraph{Deviations at Scale:}We repeat the same experiment but for a large Pythia $1.4$B model and find a relatively consistent spectrum for $10^{-4} <= \epsilon <= 10^{-3}$ as shown in Figure \ref{fig:stem_pair_2}. This is an interesting finding not reported in the literature.  What this exposes is that larger models trained on the same dataset are more stable in metrics of \textit{global} curvature compared to their smaller counterparts.

\section{Conclusion}

We presented a shard-preserving framework for computing exact Hessian  vector products under Fully Sharded Data Parallelism, enabling second-order analysis at the scale of modern foundation models. By avoiding full parameter gathers and aligning curvature computation with existing FSDP communication patterns, our method transforms the Hessian from a theoretical object into a practical systems primitive. Empirically, we demonstrated near-linear scaling across nodes and only modest constant-factor overhead relative to first-order gradient evaluation.

Building on this primitive, we performed stochastic Lanczos quadrature on language models ranging from hundreds of millions to over 100 billion parameters, producing the first Hessian spectral density estimates at true foundation-model scale. We analyzed the numerical behavior of finite-difference Hessian  vector products in this regime, characterizing truncation bias, floating-point noise amplification, and Krylov stability. These analyses yielded concrete operating regimes for the finite-difference step size and Lanczos precision, which were validated empirically across model sizes, precisions, and subsampling levels.

Our results expose qualitative failures of widely used curvature approximations. In particular, we showed that block-diagonal Hessian assumptions—central to natural-gradient methods, influence-function approximations, and post-training quantization—can incur order-one relative error and poor directional alignment even in mid-scale transformer models, with these effects persisting or worsening under fine-tuning. Direct access to the full Hessian reveals strong cross-layer coupling that is systematically discarded by such approximations.

Beyond methodology, our findings suggest that global curvature statistics stabilize with scale: larger models trained on the same data exhibit more robust Hessian spectral structure across a wide range of finite-difference step sizes and subsampling ratios. This observation, which has not been previously reported, points to a form of emergent regularity in large-scale optimization landscapes and has implications for curvature-aware optimization, monitoring, and compression.

Taken together, our work establishes that faithful Hessian analysis is not only feasible but informative at foundation-model scale. By making exact curvature accessible under realistic distributed training regimes, we open the door to principled second-order optimization, reliable influence analysis, and quantitatively grounded model compression for the next generation of large-scale models.

From a practical standpoint, our results indicate that storing Lanczos basis vectors in bfloat16—while retaining float32 arithmetic for scalar operations—is sufficient to recover stable Hessian spectra in large language models, substantially reducing memory overhead without compromising spectral fidelity.

\newpage
\section*{Broader Impact}

This work advances the practical accessibility of second-order information for large-scale machine learning systems by making faithful Hessian analysis tractable under realistic distributed training regimes. By enabling curvature-based diagnostics, optimization, and analysis at foundation-model scale, the methods introduced here may contribute to improved training stability, more efficient optimization, and better-informed model compression and deployment strategies. In particular, reliable access to Hessian spectra can support early detection of training pathologies, principled tuning of optimization hyperparameters, and more transparent assessment of model sensitivity to data and parameter perturbations.

At the same time, the techniques described in this paper are primarily infrastructural and analytical in nature. They do not directly introduce new model capabilities, nor do they alter the functional behavior of trained models. As such, we do not anticipate direct societal harms arising uniquely from the use of these methods. However, like many advances in large-scale machine learning infrastructure, improved optimization and analysis tools may indirectly lower the cost of training or refining high-capacity models, potentially accelerating the deployment of powerful systems. Responsible use therefore requires that practitioners continue to apply appropriate safeguards, evaluation protocols, and governance mechanisms when developing and deploying large models.

We also note that our findings highlight limitations of commonly used curvature approximations, such as block-diagonal Hessian assumptions, which are often employed in downstream tasks including model compression and influence analysis. While exposing these limitations can lead to more reliable methods, misuse or misinterpretation of approximate curvature information could result in misleading conclusions if applied without validation. We encourage practitioners to treat curvature-based tools as diagnostic aids rather than definitive measures, particularly when applied outside the regimes studied here.

Overall, we view this work as enabling more principled and transparent analysis of large-scale learning dynamics. By improving access to exact curvature information, it has the potential to support safer, more reliable, and better-understood foundation models when combined with responsible engineering and deployment practices.

\bibliographystyle{plainnat}
\bibliography{refs}

\appendix

\section{Understanding Finite Difference Hessian Approximations}
\begin{proof}
	Expanding $\nabla_\theta L(\theta \pm \varepsilon v)$ in a Taylor series around
	$\theta$ yields
	\[
	\nabla_\theta L(\theta \pm \varepsilon v)
	=
	\nabla_\theta L(\theta)
	\pm \varepsilon Hv
	+ \frac{\varepsilon^2}{2}\nabla_\theta(D_v^2 L(\theta))
	\pm \frac{\varepsilon^3}{6}\nabla_\theta(D_v^3 L(\theta))
	+ O(\varepsilon^4).
	\]
	Subtracting the two expansions and dividing by $2\varepsilon$ gives
	\[
	\widetilde{H}v
	=
	Hv
	+ \frac{\varepsilon^2}{6}\,\nabla_\theta(D_v^3 L(\theta))
	+ O(\varepsilon^4),
	\]
	which yields the truncation error term. Each gradient evaluation is subject to
	relative floating-point error of order $\varepsilon_{\mathrm{mach}}$, which is
	amplified by subtracting nearly equal vectors, resulting in a dominant roundoff
	contribution of magnitude
	\(
	O(\varepsilon_{\mathrm{mach}}\|\nabla_\theta L(\theta)\|/\varepsilon)
	\).
	Balancing the leading truncation and roundoff terms yields the stated optimal step
	size and error scaling.
\end{proof}

\begin{proof}
	We decompose the approximation error into a deterministic truncation (bias)
	term and a stochastic finite-precision noise term.
	
	\paragraph{Truncation (Bias) Error.}
	Since $f \in C^4(\mathbb{R}^n)$, a Taylor expansion of the directional curvature
	along the line $x + t v$ yields
	\begin{align}
		v^\top H(x + t v) v
		&=
		v^\top H(x) v
		+
		t\, D^3 f(x)[v,v,v]
		\nonumber\\
		&\quad
		+
		\frac{t^2}{2}
		D^4 f(x)[v,v,v,v]
		+
		O(t^3).
	\end{align}
	Substituting this expansion into the averaging representation of
	$H_\varepsilon(x)[v]$ and using the symmetry of the weighting kernel
	$w_\varepsilon$ eliminates the odd-order term, since
	\begin{align}
		\int_{-\varepsilon}^{\varepsilon}
		t\, w_\varepsilon(t)\,dt
		&= 0.
	\end{align}
	The leading bias term is therefore
	\begin{align}
		\frac{1}{\varepsilon}
		\int_{-\varepsilon}^{\varepsilon}
		\frac{t^2}{2}
		D^4 f(x)[v^{\otimes 4}]
		\, w_\varepsilon(t)\,dt.
	\end{align}
	A direct computation gives
	\begin{align}
		\frac{1}{\varepsilon}
		\int_{-\varepsilon}^{\varepsilon}
		t^2 w_\varepsilon(t)\,dt
		&=
		\frac{\varepsilon^2}{6},
	\end{align}
	so that
	\begin{align}
		H_\varepsilon(x)[v]
		&=
		v^\top H(x)v
		+
		\frac{\varepsilon^2}{12}
		D^4 f(x)[v^{\otimes 4}]
		+
		O(\varepsilon^4).
	\end{align}
	Consequently, the deterministic truncation error satisfies
	\begin{align}
		\bigl|
		H_\varepsilon(x)[v]
		-
		v^\top H(x)v
		\bigr|
		&\le
		C_1\, \varepsilon^2 \|D^4 f\|.
	\end{align}
	
	\paragraph{Finite-Precision and Noise Error.}
	With noisy function evaluations,
	\begin{align}
		\tilde H_\varepsilon(x)[v]
		&=
		H_\varepsilon(x)[v]
		+
		\frac{
			\eta(x+\varepsilon v)
			-
			2\eta(x)
			+
			\eta(x-\varepsilon v)
		}{\varepsilon^2}.
	\end{align}
	Since the noise terms are independent with variance $\sigma_f^2$, the variance
	of the noise contribution is
	\begin{align}
		\mathrm{Var}(\tilde H_\varepsilon)
		&=
		\frac{6\sigma_f^2}{\varepsilon^4},
	\end{align}
	implying a root-mean-square noise error of order
	\begin{align}
		O\!\left(
		\frac{\sigma_f}{\varepsilon^2}
		\right).
	\end{align}
	
	\paragraph{Combination and Optimisation.}
	Combining truncation and noise contributions yields
	\begin{align}
		\text{Error}(\varepsilon)
		&=
		O\!\left(
		\varepsilon^2 \|D^4 f\|
		\right)
		+
		O\!\left(
		\frac{\sigma_f}{\varepsilon^2}
		\right).
	\end{align}
	Balancing the two terms gives
	\(
	\varepsilon_{\mathrm{opt}}
	\asymp
	(\sigma_f / \|D^4 f\|)^{1/4}
	\),
	and substituting this value yields the stated minimal error scaling.
\end{proof}

\subsection{Finite Difference as Local Averaging}
\label{sec:localavg}
\begin{theorem}
	For $f \in C^2(\mathbb{R}^n)$,
	\begin{equation}
		H_\varepsilon(x)[v]
		=
		\frac{1}{\varepsilon}
		\int_{-\varepsilon}^{\varepsilon}
		w_\varepsilon(t)
		\;
		v^\top H(x + t v)\, v
		\;
		dt,
	\end{equation}
	where
	\begin{equation}
		w_\varepsilon(t)
		=
		1
		-
		\frac{|t|}{\varepsilon},
	\end{equation}
	and $H(x) = \nabla^2 f(x)$.
\end{theorem}

\begin{proof}
	From the integral form of Taylor's theorem,
	\begin{equation}
		f(x + \varepsilon v)
		=
		f(x)
		+
		\varepsilon \nabla f(x)^\top v
		+
		I_+(x),
	\end{equation}
	with
	\begin{equation}
		I_+(x)
		=
		\int_0^\varepsilon
		(\varepsilon - t)
		\;
		\bigl[
		v^\top H(x + t v)\, v
		\bigr]
		\;
		dt .
	\end{equation}
	
	Similarly,
	\begin{equation}
		f(x - \varepsilon v)
		=
		f(x)
		-
		\varepsilon \nabla f(x)^\top v
		+
		I_-(x),
	\end{equation}
	where
	\begin{equation}
		I_-(x)
		=
		\int_0^\varepsilon
		(\varepsilon - t)
		\;
		\bigl[
		v^\top H(x - t v)\, v
		\bigr]
		\;
		dt .
	\end{equation}
	
	Adding the two expansions yields
	\begin{equation}
		f(x+\varepsilon v)
		+
		f(x-\varepsilon v)
		-
		2 f(x)
		=
		I_+(x)
		+
		I_-(x).
	\end{equation}
	
	Combining the integrals gives
	\begin{equation}
		I_+(x)
		+
		I_-(x)
		=
		\int_0^\varepsilon
		(\varepsilon - t)
		\;
		\Xi(t)
		\;
		dt,
	\end{equation}
	with
	\begin{equation}
		\Xi(t)
		=
		v^\top H(x + t v)\, v
		+
		v^\top H(x - t v)\, v .
	\end{equation}
	
	By symmetry,
	\begin{equation}
		\int_0^\varepsilon
		(\varepsilon - t)
		\;
		\Xi(t)
		\;
		dt
		=
		\int_{-\varepsilon}^{\varepsilon}
		(\varepsilon - |t|)
		\;
		v^\top H(x + t v)\, v
		\;
		dt .
	\end{equation}
	
	Dividing by $\varepsilon^2$ gives $		\frac{
		f(x+\varepsilon v)
		-
		2 f(x)
		+
		f(x-\varepsilon v)
	}{
		\varepsilon^2
	}
	=$
	\begin{equation}
		\frac{1}{\varepsilon}
		\int_{-\varepsilon}^{\varepsilon}
		\left(
		1
		-
		\frac{|t|}{\varepsilon}
		\right)
		v^\top H(x + t v)\, v
		\;
		dt ,
	\end{equation}
	which proves the claim.
\end{proof}
\section{FSDP vs DP for small models}
\label{sec:fsdpvsdp}
To formalise the less than linear scaling for FSDP vs DP in the regime where both fit on memory, let $C$ be the single-GPU compute time for a step, $K$ the number of data-parallel ranks, $P$ the total parameter size (bytes), and $(\alpha,\beta)$ the latency/bandwidth costs of a ring all-reduce.

Under DP, a step consists of compute plus (approximately) one gradient all-reduce:
\begin{equation}
	\label{eq:tdp}
	T_{\mathrm{DP}} \;=\; \frac{C}{K} \;+\; \alpha (K-1) \;+\; 2\beta P.
\end{equation}
Under FSDP ZeRO-3, parameters are all-gathered and gradients are reduce-scattered at \emph{block granularity}. If the transformer is wrapped into $L$ FSDP units, then both forward and backward invoke per-block collectives, yielding a higher communication frequency:
\begin{equation}
	\label{eq:tfsdp}
	T_{\mathrm{FSDP}}
	\;=\;
	\frac{C}{K}
	\;+\;
	4\alpha (K-1)L
	\;+\;
	8\beta P,
\end{equation}
where constants reflect that (i) all-gathers/reduce-scatters occur repeatedly across blocks and (ii) parameters/gradients are communicated multiple times across the two passes.\footnote{The exact constants depend on overlap and implementation details; the key point is the \emph{linear dependence on $L$} in the latency term and the larger effective bandwidth term relative to DP.}

Subtracting \eqref{eq:tdp} from \eqref{eq:tfsdp} gives
\begin{equation}
	\label{eq:fsdp-slower}
	T_{\mathrm{FSDP}} - T_{\mathrm{DP}}
	\;=\;
	4\alpha (K-1)(L-1) \;+\; 6\beta P \;>\; 0,
\end{equation}
so ZeRO-3 is provably slower than DP whenever the model fits in memory and both regimes are feasible.

For deep transformers where communication is a substantial fraction of step time (e.g., $L\!\approx\!32$-$48$ on 8$\times$A100), it is realistic to have $T_{\mathrm{comp}}\!\approx\!80$ms, $T_{\mathrm{DP,comm}}\!\approx\!20$ms, and $T_{\mathrm{FSDP,comm}}\!\approx\!55$ms, yielding:
\begin{equation}
	T_{\mathrm{DP}} \approx 100\mathrm{ms},\quad
	T_{\mathrm{FSDP}} \approx 135\mathrm{ms},\quad
	\frac{T_{\mathrm{FSDP}}-T_{\mathrm{DP}}}{T_{\mathrm{DP}}}\approx 0.35.
\end{equation}
This order-of-magnitude slowdown matches what one should expect when choosing ZeRO-3 in a regime where DP is also possible. We emphasise that our contribution is not to optimise FSDP itself, but to provide an HVP/Lanczos wrapper whose cost profile \emph{tracks} the underlying FSDP backend.
\section{Appendix C: Detailed Cost Breakdown for Distributed Lanczos}
\label{app:lanczos-cost}

This appendix provides a detailed accounting of the computational and
communication costs incurred by one distributed Lanczos iteration under
FSDP ZeRO-3, complementing the high-level cost expressions given in the main text.

\paragraph{Cost of one dataset-averaged HVP.}
All additional work outside the two FSDP gradient passes consists of
linear-time shard-local vector operations over $P_{\mathrm{loc}} = P/R$
parameters (AXPY updates, copies, and scaling), plus at most one scalar
all-reduce.
Grouping these operations into $T_{\mathrm{vec}}$, the cost of one
dataset-averaged Hessian  vector product is
\begin{equation}
	\label{eq:thvp-app}
	T_{\mathrm{HvP}}
	\;=\;
	2\,T_{\mathrm{grad}}
	\;+\;
	T_{\mathrm{vec}},
	\qquad
	T_{\mathrm{vec}}
	=
	\mathcal{O}(P_{\mathrm{loc}}\,\gamma)
	+
	\mathcal{O}(T_{\mathrm{scalar}}),
\end{equation}
where $\gamma$ denotes the cost per shard-local vector operation and
$T_{\mathrm{scalar}}$ the latency of a scalar all-reduce.
Importantly, the HVP introduces no additional $O(P)$-sized collectives beyond
those already implied by FSDP execution.

\paragraph{Lanczos iteration mechanics.}
At iteration $t$, each rank stores shard-local pieces of the current and
previous Lanczos vectors $v_t$ and $v_{t-1}$.
Given the Hessian  vector product $w \leftarrow H v_t$, global recurrence
coefficients are formed via scalar reductions,
\begin{equation}
	\alpha_t = \langle w, v_t \rangle,
	\qquad
	\beta_t = \|w\|_2,
\end{equation}
followed by the standard three-term recurrence
\begin{equation}
	\label{eq:lanczos-recurrence-app}
	w \leftarrow w - \alpha_t v_t - \beta_t v_{t-1},
	\qquad
	v_{t+1} \leftarrow \frac{w}{\|w\|_2}.
\end{equation}
All vector updates are performed shard-locally.

\paragraph{Reorthogonalisation.}
When reorthogonalisation is enabled with a sliding window of size $r$,
the vector $w$ is additionally orthogonalised against up to $r$ previously
stored Lanczos vectors.
Each reorthogonalisation step requires one global dot product (scalar
all-reduce) and one shard-local AXPY update.

\paragraph{Per-iteration cost.}
Let $c_0,c_1>0$ be constants capturing the number of shard-local vector
operations required per iteration, excluding the Hessian  vector product.
Then the cost of one Lanczos iteration with window size $r$ is
\begin{equation}
	\label{eq:tlanczos-app}
	T_{\mathrm{Lanczos\ iter}}(r)
	\;=\;
	T_{\mathrm{HvP}}
	\;+\;
	(2+r)\,T_{\mathrm{scalar}}
	\;+\;
	\big(c_0 + c_1 r\big)\,P_{\mathrm{loc}}\,\gamma.
\end{equation}
Substituting $P_{\mathrm{loc}} = P/R$ yields the expanded expression
\begin{equation}
	\label{eq:tlanczos-expanded-app}
	T_{\mathrm{Lanczos\ iter}}(r)
	=
	2\,T_{\mathrm{grad}}
	+
	T_{\mathrm{vec}}
	+
	(2+r)\,T_{\mathrm{scalar}}
	+
	\big(c_0 + c_1 r\big)\,\frac{P}{R}\,\gamma.
\end{equation}

\paragraph{End-to-end complexity.}
For stochastic Lanczos quadrature (SLQ) with $m$ Krylov steps and $s$ probe
vectors, the total runtime is approximately
\begin{equation}
	T_{\mathrm{SLQ}}
	\;\approx\;
	s\,m\,T_{\mathrm{Lanczos\ iter}}(r)
	+
	T_{\mathrm{post}},
\end{equation}
where $T_{\mathrm{post}}$ accounts for tridiagonal eigendecompositions and
quadrature evaluation and is typically negligible relative to the $sm$
HVP calls at scale.
Conjugate-gradient-based inverse-HVP solvers admit an analogous cost
structure, differing only in constant factors.

\end{document}